\documentclass[]{article}
\usepackage{amsmath,amsfonts,amssymb,amsthm,amstext,graphicx,subfigure,color,bbm}
\usepackage{subfigure}
\usepackage[]{hyperref}
\usepackage{multirow}
\hypersetup{pdftitle={Brain MRI Segmentation with Fast and Globally Convex Multiphase Active Contours},
colorlinks=true,citecolor=blue,linkcolor=red}
\newcommand{\abs}[1]{\left\vert #1 \right\vert}

\newtheorem{theorem}{Theorem}

\newtheorem{remark}{Remark}
\begin{document}
\title{Brain MRI Segmentation with Fast and Globally Convex Multiphase Active Contours}
\author{Juan C. Moreno\footnote{Corresponding author. IT, Department of Computer Science, University of Beira Interior, 6201--001, Portugal
E-mail: jmoreno@ubi.pt},
\and
V. B. S. Prasath\footnote{Department of Computer Science,
University of Missouri-Columbia, MO 65211 USA. E-mail: prasaths@missouri.edu}, 
\and 
Hugo Proen\c{c}a\footnote{IT, Department of Computer Science, University of Beira Interior, 6201--001, Portugal. E-mail: hugomcp@ubi.pt}, 
\and 
K. Palaniappan\footnote{Department of Computer Science,
University of Missouri-Columbia, MO 65211 USA.
E-mail: palaniappank@missouri.edu}}

\date{}
\maketitle


\begin{abstract}

Multiphase active contour based models are useful in identifying multiple regions with different characteristics such as the mean values of regions. This is relevant in brain magnetic resonance images (MRIs), allowing the differentiation of white matter against gray matter. We consider a well defined globally convex formulation of Vese and Chan multiphase active contour model for segmenting brain MRI images. A well-established theory and an efficient dual minimization scheme are thoroughly described which guarantees optimal solutions and provides stable segmentations. Moreover, under the dual minimization implementation our model perfectly describes disjoint regions by avoiding local minima solutions. Experimental results indicate that the proposed approach provides better accuracy than other related multiphase active contour algorithms even under severe noise, intensity inhomogeneities, and partial volume effects.

\end{abstract}
\textbf{Keywords}: Image segmentation, active contours, multiphase, globally convex, dual formulation, brain MRI.


\section{Introduction}\label{sec:intro}

The aim of image segmentation is to obtain meaningful partitions of an input image into a finite number of disjoint homogeneous objects. Active contour models are popular in the regard. Chan and Vese~\cite{CV01} proposed an active contour without edges scheme based on the classical work of Mumford and Shah~\cite{MS89} variational energy minimization model. Since biomedical images typically have multiple regions of interest with different characteristics, deriving a multiphase active contour scheme for efficient segmentation is an important area of research in image processing~\cite{CS00,VC02,KeeganSC12}.

In MRI (magnetic resonance image) images, segmentations based on active contours have been used with traditional level set method~\cite{OS88}.  Active contours can also be improved using region information~\cite{DrapacaMultiphaseBrainMRI05}, salient features~\cite{KohISBI11} or mathematical morphology~\cite{GuiISBI11} etc. Traditionally these
schemes use a gradient descent formulation to implement the non-convex energy minimization and can stuck in undesired local minima thereby lead to erroneous segmentations. Moreover, traditional level set based implementation is prone to slower convergence due to the well-known re-initialization requirement and discretization errors. More recently quite a lot of interest is being shown in techniques that can obtain a general convex formulation for active contours schemes based on energy minimization which can alleviate the problem of local minima at the same time focussing on the computational complexity~\cite{DoganMorinMSshape08,BCB10,BaeYTai11,BrownCB12,JCthesis12, KangMarch13}. Among other techniques for MRI image segmentation, we mention fuzzy C-means based models~\cite{ahmed2002modified,liew2003adaptive,chuang2006fuzzy}, fuzzy connectedness~\cite{Zhuge20091095}, automatic labeling~\cite{fischl2002whole}, adaptive expectation-maximization (EM)~\cite{wells1996adaptive}, Bayesian EM~\cite{marroquin2002accurate}, hidden Markov model EM~\cite{zhang2001segmentation}, kernel clustering~\cite{liao2008mri}, optimum-path clustering~\cite{Cappabianco20121047}, anisotropic diffusion combined with classical snakes model~\cite{atkins1998fully}, discriminant analysis~\cite{amato2003segmentation}, and neural networks~\cite{shen2005mri}. We also refer to~\cite{pal1993review,pham2000current,MT10} for reviews about segmentation for medical images in general and~\cite{bezdek1993review} for MR images in particular.  The area of MRI image segmentation has seen tremendous research activity and a more detailed review in this particular field can be found in~\cite{balafar2010review}.

\begin{figure}
\centering
\subfigure[]{\includegraphics[width=2cm,height=2cm]{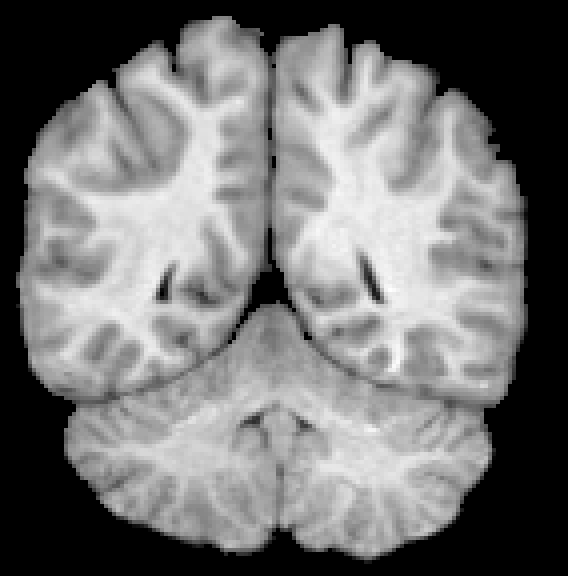}}
\subfigure[]{\includegraphics[width=2cm,height=2cm]{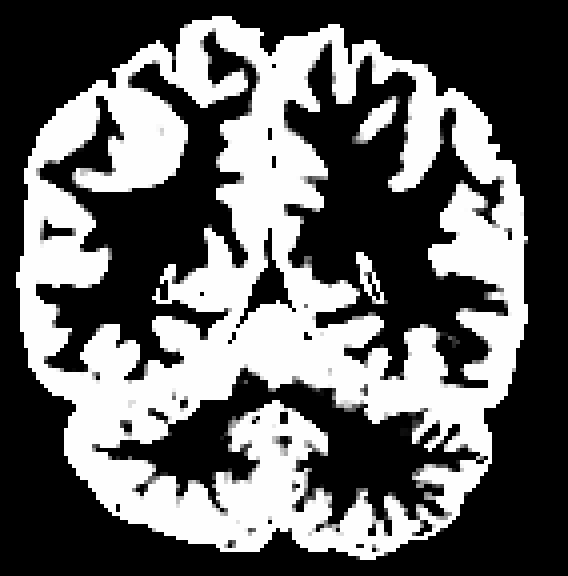}}
\subfigure[]{\includegraphics[width=2cm,height=2cm]{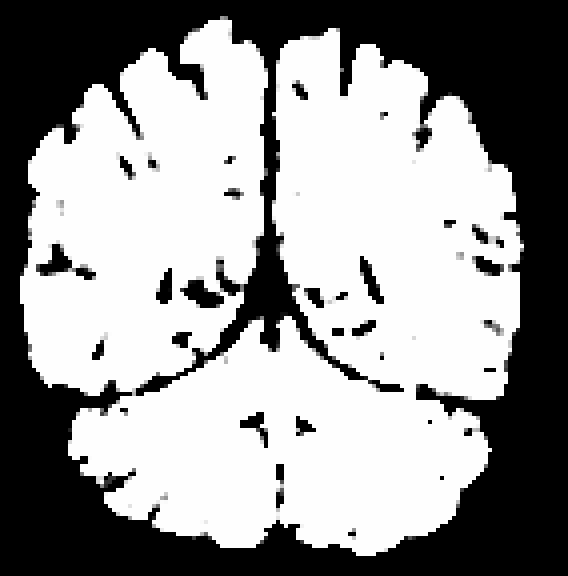}}
\subfigure[]{\includegraphics[width=2cm,height=2cm]{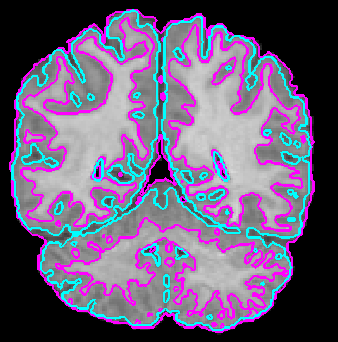}}
\subfigure[]{\includegraphics[width=2cm,height=2cm]{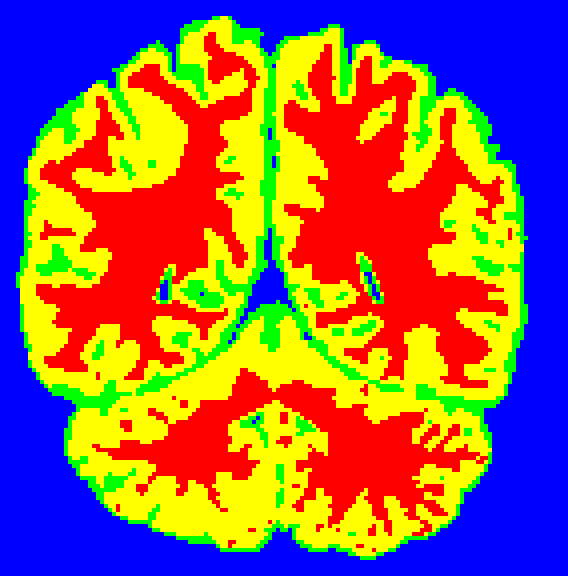}}\\
\caption{\footnotesize{Our fast and automatic four phase image segmentation scheme provides a better segmentations for brain MRI images, it differentiates the gray matter from the surrounding white region clearly.
            (a) Input image with noise level $n=5\%$,
            (b) \& (c) show final binary segmentations obtained by thresholding the relaxed functions $u_1,u_2$ at $0.5$,
            (d) final segmentation result showing the contours superimposed on the input image,
            (e) color coded visualization of the obtained segmentation result.}}\label{fig:teaser}
\end{figure}

In this paper, we consider a globally convex version of the four phase piecewise constant energy functional following the seminal work of Chan et al~\cite{CE06}. By deriving an approximate novel convex functional we change the original formulation into a binary segmentation problem and utilize a dual minimization to solve the relaxed formulation~\cite{Ch04}. The proposed global methodology avoids the level set re-initialization constraint and other ad-hoc techniques~\cite{MerinoISBI10} used for fixing level set active contour movements throughout the iterations. The proposed approach is used to obtain white matter and gray matter partitions on brain MRI images as can be seen for example in Figure~\ref{fig:teaser}. Our scheme does not involve level sets or re-initialization and instead relies on the relaxed globally convex formulation of the Vese and Chan multiphase active contours. Comparison results on the different image sets with varying noise and inhomogeneities show that we can obtain better results than traditional level set multiphase schemes~\cite{VC02,AyedMitiche06,AyedMiticheBelhadj06,AyedMitiche08,BenSalahMiticheTIP10} and primal-dual approach of~\cite{ChambollePock11}. Moreover, compared to these traditional level set based implementations we achieve faster convergence due to the usage of efficient alternating dual minimization. The proposed approach is general in the sense that we can add domain specific knowledge to improve such active contour schemes further for various tasks~\cite{LK08,Zhang2011256,FigueiredoJuan12,FigueiredoJ12,PrasathBunyak12,PrasathPalG12}.

The main contribution of our work is two-fold: 1) a fast four phase active contour model using a relaxed globally convex minimization approximation; 2) using an efficient dual minimization based implementation for performing segmentation on MRI images. The rest of the paper is organized as follows. Section~\ref{sec:multi} introduces the multiphase variational active contour scheme and provides a globally convex formulation. Section~\ref{sec:exper} illustrates the segmentation results on various Brain MRI images including comparison of different schemes. Finally, Section~\ref{sec:concl} concludes the paper.
\section{Multiphase active contours model}\label{sec:multi}

We first recall the multiphase formulation of Vese and Chan~\cite{VC02}
and restrict ourselves to the piecewise constant four phase model since the general case can be derived similarly. Let
$\phi_2,\phi_2:\Omega\subset\mathbb{R}^2\to \mathbb{R}$ be the two
level sets. $H_1 = H(\phi_{1})$, $H_2 = H(\phi_{2})$ and $\tilde
H_1 = 1-H(\phi_{1})$, $\tilde H_2=1-H(\phi_{2})$, where $H$ is the
Heaviside function, representing four regions. Our goal is to solve a minimization problem
\begin{gather}\label{E:origmin}
\min_{(\mathbf{c},\Phi)}F(\mathbf{c},\Phi)
\end{gather}
with
\begin{equation*}
\begin{aligned}
F(\mathbf{c},\Phi)&=\mu_{1}\int_{\Omega}\delta(\phi_{1})|\nabla\phi_{1}|\,d\mathbf{x}
+\mu_{2}\int_{\Omega}\delta(\phi_{2})|\nabla\phi_{2}|\,d\mathbf{x}\\
&+\lambda_{11}\int_{\Omega}(I-c_{11})^{2}
H_1\,H_2\,d\mathbf{x}
+\lambda_{10}\int_{\Omega}(I-c_{10})^{2}H_1\,\tilde H_2\,d\mathbf{x}\\
&+\lambda_{01}\int_{\Omega}(I-c_{01})^{2}\tilde
H_1\,H_2\,d\mathbf{x}
+\lambda_{00}\int_{\Omega}(I-c_{00})^{2}\tilde H_1\,\tilde
H_2\,d\mathbf{x}
\end{aligned}
\end{equation*}
where $\Phi=(\phi_{1},\phi_{2})$, and the constant mean values
$\mathbf{c}=(c_{11}, c_{10}, c_{01}, c_{00})$ can be derived as
\[c_{11}=\frac{\int_{\Omega} I\,H_1\,H_2\,d\mathbf{x}}{\int_{\Omega}H_1\,H_2\,d\mathbf{x}},\quad
c_{10}=\frac{\int_{\Omega}I\,H_1\,\tilde
H_2\,d\mathbf{x}}{\int_{\Omega}H_1\,\tilde H_2\,d\mathbf{x}}\]
\[c_{01}=\frac{\int_{\Omega} I\,\tilde H_1\,H_2\,d\mathbf{x}}{\int_{\Omega}\tilde H_1\,H_2\,d\mathbf{x}},\quad
c_{00}=\frac{\int_{\Omega} I\,\tilde H_1\,\tilde
H_2\,d\mathbf{x}}{\int_{\Omega}\tilde H_1\,\tilde
H_2\,d\mathbf{x}}\] 
Note the the zero level sets $\phi_i = 0$, $i=1,2$, represent object boundaries and the mean values $\mathbf{c}$ represent the expected average pixel values in these objects. Vese and Chan~\cite{VC02} used the corresponding gradient descent equations to implement the active contours~\cite{OS88}. In the numerical implementation of the above PDEs, a non-compactly supported, smooth approximation of the Heaviside function $H_{\epsilon}(\mathbf{x})$, such that $H_{\epsilon}(\mathbf{x}) \rightarrow H(\mathbf{x})$ as $\epsilon\rightarrow 0$ is utilized. Since the above minimization~\eqref{E:origmin} is non-convex the time discretized gradient descent PDEs usually require large iterations and small time steps to convergence (typically in $100$'s of iterations). Moreover, the final segmentation result may not correspond to the global minimum of the energy function as the gradient descent scheme can be stuck at a local minima of the corresponding energy functional given in Eqn.~\eqref{E:origmin}.

We briefly recall the corresponding gradient descent equations (time dependent Euler-Lagrange equations of Eqn.~\eqref{E:origmin}) for the level sets functions $\phi_1$ and $\phi_2$,
\begin{equation}\label{EL1}
\phi_{1t}=\delta(\phi_{1})\left(\mu_{1}\,div\left(\frac{\nabla\phi_{1}}{|\nabla\phi_{1}|}\right)-r_{1}(\mathbf{c},H_{2})\right)
\end{equation}
and 
\begin{equation}\label{EL2}
\phi_{2t}=\delta(\phi_{2})\left(\mu_{2}\,div\left(\frac{\nabla\phi_{2}}{|\nabla\phi_{2}|}\right)-r_{2}(\mathbf{c},H_{1})\right)
\end{equation}
respectively. Here, the image fitting terms are given by,
\begin{align*}
r_{1}(\mathbf{c},H_{2})&=(\lambda_{11}(I-c_{11})^{2}-\lambda_{01}(I-c_{01})^{2})H_{2} \\
&+ (\lambda_{10}(I-c_{10})^{2}-\lambda_{00}(I-c_{00})^{2})\tilde{H}_{2}\\
r_{2}(\mathbf{c},H_{1})&=(\lambda_{11}(I-c_{11})^{2}-\lambda_{10}(I-c_{10})^{2})H_{1} \\
&+ (\lambda_{01}(I-c_{01})^{2}-\lambda_{00}(I-c_{00})^{2})\tilde{H}_{1}.
\end{align*}

Following, Chan et al~\cite{CE06}, we derive a relaxed energy
minimization formulation by dropping the dirac delta function
($\delta(\phi)$ in~\eqref{EL1} and~\eqref{EL2}) to obtain,
\begin{gather}\label{E:gmin}
\min_{(\Phi, \mathbf{c})}\mathcal{F}(\Phi, \mathbf{c})
\end{gather}
with
\begin{equation*}
\mathcal{F}(\mathbf{c},\Phi)=\mu_{1}\int_{\Omega}|\nabla\phi_{1}|\,d\mathbf{x}+\mu_{2}\int_{\Omega}|\nabla\phi_{2}|\,d\mathbf{x}+\int_{\Omega}r_{1}(\mathbf{c},H_{2})\phi_{1}+\int_{\Omega}r_{2}(\mathbf{c},H_{1})\phi_{2}\,d\mathbf{x}\\
\end{equation*}
Then correspondingly we can derive an energy functional which does not depend on regularized Heaviside functions. Thus, we can solve the following globally convex energy minimization problem,
\begin{eqnarray}\label{E:ourmin1}
\min_{\mathbf{u}=(u_{1},u_{2})\in\{0,1\}^{2}}\mathcal{G}(\mathbf{c},\mathbf{u})
\end{eqnarray}
with
\begin{equation*}
\begin{aligned}
\mathcal{G}(\mathbf{c},\mathbf{u})&=
\mu_{1}\int_{\Omega}|\nabla u_{1}|\,d\mathbf{x}
+\mu_{2}\int_{\Omega}|\nabla u_{2}|\,d\mathbf{x}\\
&+\lambda_{11}\int_{\Omega}(I-c_{11})^2 u_1u_2\,d\mathbf{x}+ \lambda_{01}\int_{\Omega}(I-c_{01})^2 (1-u_{1})u_{2}\,d\mathbf{x}\\
&+ \lambda_{10}\int_{\Omega}(I-c_{10})^2 u_1(1-u_2)\,d\mathbf{x}+\lambda_{00}\int_{\Omega}(I-c_{00})^2(1-u_{1})(1-u_2)\,d\mathbf{x},
\end{aligned}
\end{equation*}
where Heaviside functions are replaced by $\mathbf{u}=(u_{1},u_{2})\in\{0,1\}^{2}$ which are known as \textit{binary partitioning functions}. The above modified minimization problem~\eqref{E:ourmin1} can further be relaxed to the set of functions $\mathbf{u}=(u_{1},u_{2})\in[0,1]^{2}$ in order to solve a convex minimization problem. That is, the binary partitioning functions based energy minimization becomes,
\begin{gather}\label{E:ourmin2}
\min_{\mathbf{u}=(u_{1},u_{2})\in
[0,1]^{2}}\mathcal{G}(\mathbf{c},\mathbf{u}).
\end{gather}
The following theorem provides the guarantee of finding a global minimizer for the derived functional~\eqref{E:ourmin1} in terms of the relaxed version in~\eqref{E:ourmin2}. We follow arguments similar to the work of Chan et al~\cite{CE06} and~\cite{JCthesis12} to prove the following result.
\begin{theorem}
For any  $c_{11}, c_{10}, c_{01}, c_{00}\in\mathbb{R}$, a global minimizer for
$\mathcal{G}(\mathbf{c}, \cdot)$ in \eqref{E:ourmin1} can be found by carrying
out the convex minimization problem \eqref{E:ourmin2}.
\end{theorem}
\begin{proof}
We use the standard notation for functions of bounded variation \cite{AttouchB06}. Since $\mathbf{u}\in [0,1]^{2}$, it follows from the standard total variation based Coarea Formula,
\begin{eqnarray*}
\int_{\Omega}|\nabla u_{1}|\,d\mathbf{x}=\int_{0}^{1}\int_{0}^{1} Per\left(\{\mathbf{x}\in\Omega\,:\, u_{1}(\mathbf{x})> \zeta_{1}\}; \Omega\right)\,d\zeta_{1}\,d\zeta_{2}
\end{eqnarray*} 
and 
\begin{eqnarray*}
\int_{\Omega}|\nabla u_{2}|\,d\mathbf{x}=\int_{0}^{1}\int_{0}^{1} Per\left(\{\mathbf{x}\in\Omega\,:\, u_{2}(\mathbf{x})> \zeta_{2}\}; \Omega\right)\,d\zeta_{1}\,d\zeta_{2}.
\end{eqnarray*} 
For the image fitting term,
\begin{equation*}
\begin{aligned}
\int_{\Omega}(u-c_{11})^{2}u_{1}u_{2}\,d\mathbf{x}&=\int_{\Omega}(u-c_{11})^{2}\prod_{i=1}^2\Bigg(\int_{0}^{1}\mathbbm{1}_{\{\mathbf{u}\in\Omega\,:\, u_{i}>\zeta_{i}\}}\,d\zeta_{i}\Bigg)\,d\mathbf{x}\\
&=\int_{0}^{1}\int_{0}^{1}\int_{\Omega}(u-c_{11})^{2}\mathbbm{1}_{\{\mathbf{x}\in\Omega\,:\, u_{1}>\zeta_{1}\}}\mathbbm{1}_{\{\mathbf{x}\in\Omega\,:\, u_{2}>\zeta_{2}\}}\, d\mathbf{x}\,d\zeta_{1}\,d\zeta_{2}.
\end{aligned}
\end{equation*}
Further similar computations yield,
\begin{equation*}
\begin{aligned}
\int_{\Omega}(u-c_{01})^{2}&(1-u_{1})u_{2}\,d\mathbf{x}=\\
&=\int_{0}^{1}\int_{0}^{1}\int_{\Omega}(u-c_{01})^{2}(1-\mathbbm{1}_{\{\mathbf{x}\in\Omega\,:\, u_{1}>\zeta_{1}\}})\mathbbm{1}_{\{\mathbf{x}\in\Omega\,:\, u_{2}>\zeta_{2}\}}\, d\mathbf{x}\,d\zeta_{1}\,d\zeta_{2},
\end{aligned}
\end{equation*}
\begin{equation*}
\begin{aligned}
\int_{\Omega}(u-c_{10})^{2}&u_{1}(1-u_{2})\,d\mathbf{x}=\\
&=\int_{0}^{1}\int_{0}^{1}\int_{\Omega}(u-c_{10})^{2}\mathbbm{1}_{\{\mathbf{x}\in\Omega\,:\, u_{1}>\zeta_{1}\}}(1-\mathbbm{1}_{\{\mathbf{x}\in\Omega\,:\, u_{2}>\zeta_{2}\}})\, d\mathbf{x}\,d\zeta_{1}\,d\zeta_{2},
\end{aligned}
\end{equation*}
\begin{equation*}
\begin{aligned}
\int_{\Omega}(u-&c_{00})^{2}(1-u_{1})(1-u_{2})\,d\mathbf{x}=\\
&=\int_{0}^{1}\int_{0}^{1}\int_{\Omega}(u-c_{00})^{2}(1-\mathbbm{1}_{\{\mathbf{x}\in\Omega\,:\, u_{1}>\zeta_{1}\}})(1-\mathbbm{1}_{\{\mathbf{x}\in\Omega\,:\, u_{2}>\zeta_{2}\}})\, d\mathbf{x}\,d\zeta_{1}\,d\zeta_{2}.
\end{aligned}
\end{equation*}
 Defining $\mathbbm{1}_{\mathbf{u}}:=(\mathbbm{1}_{\{\mathbf{x}\in\Omega\,:\, u_{1}>\zeta_{1}\}},\mathbbm{1}_{\{\mathbf{x}\in\Omega\,:\, u_{2}>\zeta_{2}\}})$, it follows that 
\begin{equation*}\label{E:relations}
\mathcal{G}(\mathbf{c},\mathbf{u})=\int_{0}^{1}\int_{0}^{1}\mathcal{G}(\mathbf{c},\mathbbm{1}_{\mathbf{u}})\,d\zeta_{1}\,d\zeta_{2}=\int_{0}^{1}\int_{0}^{1}F(\mathbf{c},\mathbf{u}-\mbox{\boldmath$\zeta$})\,d\zeta_{1}\,d\zeta_{2},
\end{equation*}
for {\it a.e.} $\mbox{\boldmath$\zeta$}=(\zeta_{1},\zeta_{2})\in [0,1]^{2}$. Thus, it follows from the above equations that if  $\mathbf{u}$ is a minimizer of the convex relaxed problem~\eqref{E:ourmin2}, then for {\it a.e.} $\mbox{\boldmath$\zeta$}\in [0,1]^{2}$, the function $\mathbf{w}_{1}=\mathbbm{1}_{\mathbf{u}}$ is a minimizer of the problem \eqref{E:ourmin1}. 
\end{proof}
\begin{remark}
Note also that $\mathbf{w}_{2}=\mathbf{u}-\mbox{\boldmath$\zeta$}$ is a solution of the original Vese and Chan minimization problem~\eqref{E:origmin}. This shows that the relaxed convex minimization problem is equivalent to the original Vese and Chan piecewise constant multiphase formulation~\eqref{E:origmin}, we refer to Chan et al~\cite{CE06} for more details.
\end{remark}
The final segmentation is obtained by thresholding the functions $u_1$ and $u_2$ with any number in the interval $(0,1)$ for example at $0.5$, as shown in Figure~\ref{fig:teaser}(b) and (c). Note that the above modified minimization model does not involve level sets and thus can be solved efficiently. Further, we can prove that the above relaxed minimization problem can be solved in a binary variable minimization formulation to find a global minimum. The existence of minimizers of the modified energy $\mathcal{G}$ given in Eqn.~\eqref{E:ourmin2}  is proved using the theory of functions of bounded variation (BV) space~\cite{GI84}.
\begin{theorem}For a given input gray scale image $I\in L^{\infty}(\Omega)$, there exists a minimizer for the  functional $\mathcal{G}$ in~\eqref{E:ourmin2} in $\mathbb{R}^{4}\times BV_{[0,1]}(\Omega)^{2}.$  
\end{theorem}
\begin{proof}
Let $m:=\inf \mathcal{G}(\mathbf{c},\mathbf{u})$ and $\{(\mathbf{c}^{k},\mathbf{u}^{k})\}_{k=1}^{\infty}\subseteq\mathbb{R}^{4}\times BV_{[0,1]}(\Omega)^{2}$ be a minimizer sequence for the energy $\mathcal{G}$, {\it i.e.}, 
\begin{equation*}\label{Ec3:s3:e3}
\mathcal{G}(\mathbf{c}^{k},\mathbf{u}^{k})\xrightarrow{k\rightarrow{} \infty \;\; }m.
\end{equation*}
Since $\{\mathbf{u}^{k}\}_{k=1}^{\infty}$ is bounded in $BV_{[0,1]}(\Omega)^{2}$, there is a subsequence also denoted by $\{\mathbf{u}^{k}\}_{k=1}^{\infty}$,  strongly convergent  to an
element  $\mathbf{u}^{*} \in L^{1}(\Omega)^{2}$. Furthermore, $\mathbf{u}^{*}\in L_{[0,1]}^{1}(\Omega)^{2}$. Therefore, it follows that $\mathbf{u}^{*}\in BV_{[0,1]}(\Omega)^{2}$ and
\begin{equation}\label{Ec3:s3:e555}
\int_{\Omega}|Du_{i}^{*}|\,\mathrm{d}\mathbf{x}\leq\liminf_{k\rightarrow\infty}\int_{\Omega} |Du_{i}^{k}| \,\mathrm{d}\mathbf{x} \quad(\mbox{with $i=1,2$}).
\end{equation}

Now, considering $\mathcal{G}$ as a function of $\mathbf{c}$, its minimization brings the following two equations,
\[c_{11}^{k}=\frac{\int_{\Omega} I\,u_{1}^{k}\,u_{2}^{k}\,d\mathbf{x}}{\int_{\Omega}u_{1}^{k}\,u_{2}^{k}\,d\mathbf{x}},\quad
c_{10}^{k}=\frac{\int_{\Omega}I\,u_{1}^{k}\,(1-u_{2}^{k})\,d\mathbf{x}}{\int_{\Omega}u_{1}^{k}\,(1-u_{2}^{k})\,d\mathbf{x}}\]
\[c_{01}^{k}=\frac{\int_{\Omega} I\,(1-u_{1}^{k})\,u_{2}^{k}\,d\mathbf{x}}{\int_{\Omega}(1-u_{1}^{k})\,u_{2}^{k}\,d\mathbf{x}},~
c_{00}^{k}=\frac{\int_{\Omega} I\,(1-u_{1}^{k})\,(1-u_{2}^{k})\,d\mathbf{x}}{\int_{\Omega}(1-u_{1}^{k})\,(1-u_{2}^{k})\,d\mathbf{x}}.\] 
Since $I\in L^{\infty}(\Omega)$, it follows $\{\mathbf{c}^{k}\}_{k=1}^{\infty}$ is uniformly bounded. Hence, there is a subsequence also denoted by $\{\mathbf{c}^{k}\}_{k=1}^{\infty}\subset\mathbb{R}^{2n}$ and a constant vector $\mathbf{c}^{*}\in\mathbb{R}^{2n}$ such that 
\begin{equation*}
\mathbf{c}^{k}\xrightarrow{k\rightarrow{} \infty \;\; }\mathbf{c}^{*}.
\end{equation*}
Then, from Fatou's lemma we get  for the suitable sequence $\{(\mathbf{c}^{k},\mathbf{u}_{k})\}_{k=1}^{\infty}$:
\begin{equation*}
\mathcal{G}(\mathbf{c}^{*},\mathbf{u}^{*})\leq\liminf_{k\to\infty}\mathcal{G}(\mathbf{c}^{k},\mathbf{u}_{k})=m,
\end{equation*}
{\it i.e.}, $(\mathbf{c}^{*},\mathbf{u}^{*})$ is a minimizer of the functional $\mathcal{G}$.
\end{proof}
Note that the $\mathbf{c}$ values given in the above theorem are computed in the numerical scheme based on a dual minimization formulation which we describe next.
\subsection{Implementation details}\label{ssec:imple}

The four phase convex minimization problem in~\eqref{E:ourmin2} is
solved in an alternating fashion for the image variables $(u_1,u_2)$:
\begin{itemize}
\item First fix $u_{2}$, and solve for $u_1$:
\begin{equation*}
\min\limits_{u_{1}\in [0,1]}\left\{\mathcal{G}_{1}(u_{1})=
\int_{\Omega}|\nabla u_{1}| \,\mathrm{d}\mathbf{x} +
\int_{\Omega}r_{1}(\mathbf{c},u_{2})
u_{1}\,\mathrm{d}\mathbf{x}\right\}.
\end{equation*}

\item Then fix $u_1$, and solve for $u_2$:
\begin{equation*}
\min\limits_{u_{2}\in [0,1]}\left\{\mathcal{G}_{2}(u_{2})=
\int_{\Omega}|\nabla u_{2}| \,\mathrm{d}\mathbf{x} +
\int_{\Omega}r_{2}(\mathbf{c},u_{1})
u_{2}\,\mathrm{d}\mathbf{x}\right\},
\end{equation*}
\end{itemize}
where the image region fitting terms are given by, 
\begin{equation*}
\begin{aligned}
r_{1}(\mathbf{c},u_{2})&=(\lambda_{11}(I-c_{11})^{2}-\lambda_{01}(I-c_{01})^{2})u_{2}\\
&+ (\lambda_{10}(I-c_{10})^{2}-
\lambda_{00}(I-c_{00})^{2})(1-u_{2}),\\
r_{2}(\mathbf{c},u_{1})&=(\lambda_{11}(I-c_{11})^{2}-\lambda_{10}(I-c_{10})^{2})u_{1}\\
&+ (\lambda_{01}(I-c_{01})^{2}-
\lambda_{00}(I-c_{00})^{2})(1-u_{1}).
\end{aligned}
\end{equation*}
To solve the above convex optimization
problems we use the Chambolle's dual formulation~\cite{Ch04,BE07} of the
total variation regularization function which occurs as the first term in the energy functional in Eqn.~\eqref{E:ourmin2}. Thus, the new
unconstrained minimization problems to consider are (for $j=1,2$):
\begin{equation*}
\min_{u_{j},v_{j}}\Bigg\{\int_{\Omega}\,|\nabla u_{j} |\,
d\mathbf{x} + \frac{1}{2\theta_{j}}\|u_{j}-v_{j}\|^{2}_{L^{2}(\Omega)}+\int_{\Omega}(r_{j}(\mathbf{c},u_{i}) v_{j}+
\alpha_{j}\nu(v_{j}))\, d\mathbf{x}\Bigg\},
\end{equation*}
where  $j=1,2$ and $j\neq i$, $\theta_{j}$ is chosen to be small
and $\nu(\xi):=\max\{0,2|\xi - \frac{1}{2}|-1\}$ and
$\alpha_{j}>\frac{1}{2}\|r_{j}\|_{L^{\infty}(\Omega)}$. We solve
the above by further splitting into two sub-problems:
\begin{enumerate}
\item Solve for $u_j$:
\begin{equation*}
\min_{u_j}\left\{\int_{\Omega}\,|\nabla u_{j} |\, d\mathbf{x}+
\frac{1}{2\theta_{j}}\|u_{j}-v_{j}\|^{2}_{L^{2}(\Omega)}\right\}
\end{equation*}
The solution is given by \[u_{j} = v_{j} - \theta_{j} div\,
\mathbf{p}_{j}.\] The vector
$\mathbf{p}_{j}=(p_{j_{1}},p_{j_{2}})$ satisfy the equation
\[\nabla(\theta_{j} div\,\mathbf{p}_{j}  -
v_{j})-|\nabla(\theta_{j} div\, \mathbf{p}_{j}- v_{j})
|\mathbf{p}_{j}=0\] and it is solve by a fixed point method:
$\mathbf{p}_{j}^{0} = 0$ and \[\mathbf{p}_{j}^{n+1} =
\frac{\mathbf{p}_{j}^{n}+\delta t \nabla(
div(\mathbf{p}_{j}^{n}) - v_{j}/\theta_{j}) }{1+\delta
t|\nabla(div(\mathbf{p}_{j}^{n}) -
v_{j}/\theta_{j})|}.\]

\item Solve for the auxiliary variable $v_{j}$:
\begin{equation*}
\min_{v_j}\Bigg\{\frac{1}{2\theta_{j}}\|u_{j}-v_{j}\|^{2}_{L^{2}(\Omega)}+\int_{\Omega}(r_{j}(\mathbf{c},u_{i})
v_{j}+ \alpha_{j}\nu(v_{j})) d\mathbf{x}\Bigg\},
\end{equation*}
for which the solution is given by:
\[v_{j}=\min\big\{\max\left(u_{j}(\mathbf{x})-\theta_{j}r_{j}(\mathbf{c},u_{i}),0\right),1\big\}.\]
\end{enumerate}
Furthermore, at every few iterations the vector $\mathbf{c}$ is updated according to the following equations:
\[c_{11}=\frac{\int_{\Omega} I\,u_{1}\,u_{2}\,d\mathbf{x}}{\int_{\Omega}u_{1}\,u_{2}\,d\mathbf{x}},\quad
c_{10}=\frac{\int_{\Omega}I\,u_{1}\,(1-u_{2})\,d\mathbf{x}}{\int_{\Omega}u_{1}\,(1-u_{2})\,d\mathbf{x}}\]
\[c_{01}=\frac{\int_{\Omega} I\,(1-u_{1})\,u_{2}\,d\mathbf{x}}{\int_{\Omega}(1-u_{1})\,u_{2}\,d\mathbf{x}},~
c_{00}=\frac{\int_{\Omega} I\,(1-u_{1})\,(1-u_{2})\,d\mathbf{x}}{\int_{\Omega}(1-u_{1})\,(1-u_{2})\,d\mathbf{x}}.\] 
The computation of $\mathbf{c}$ values are similar to the ones in Vese and Chan model~\cite{VC02} (see Eqn.~\ref{E:origmin}) except that they are now based on the binary partitioning functions and does not involve computing regularized Heaviside functions.  We refer to~\cite{Ch04} for more details on this particular form of dual minimization and the motivation for the fixed point method used to derive the solution for the auxiliary variable in the second step.
\section{Experimental results}\label{sec:exper}
\begin{figure}
\centering
\subfigure[]{\includegraphics[width=2.2cm,height=2.4cm]{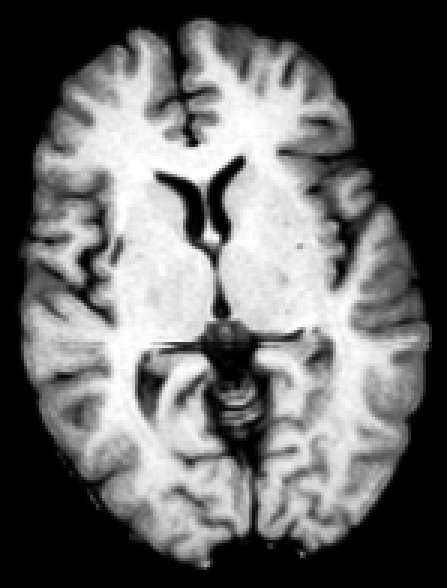}}
\subfigure[]{\includegraphics[width=2.2cm,height=2.4cm]{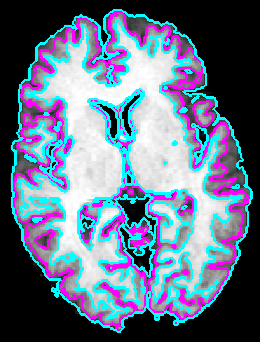}}
\subfigure[]{\includegraphics[width=2.2cm,height=2.4cm]{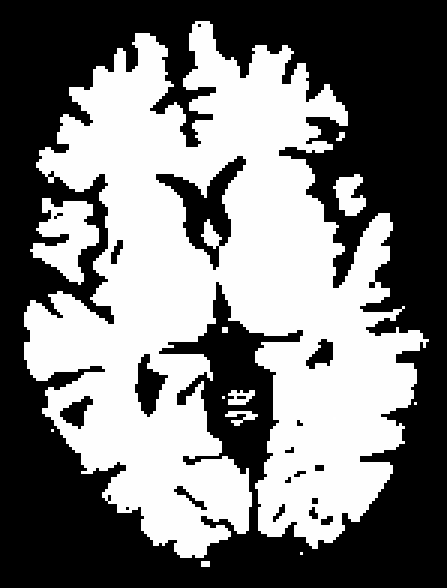}}
\subfigure[]{\includegraphics[width=2.2cm,height=2.4cm]{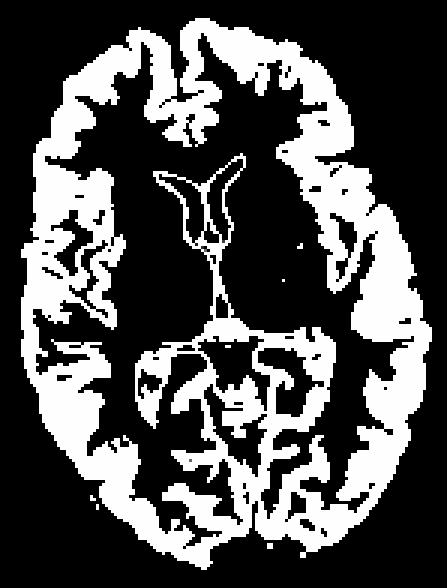}}
\subfigure[]{\includegraphics[width=2.2cm,height=2.4cm]{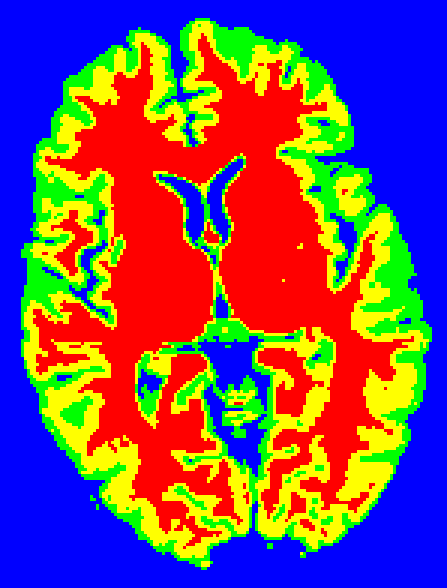}}\\
\subfigure[]{\includegraphics[width=4.5cm,height=3.1cm]{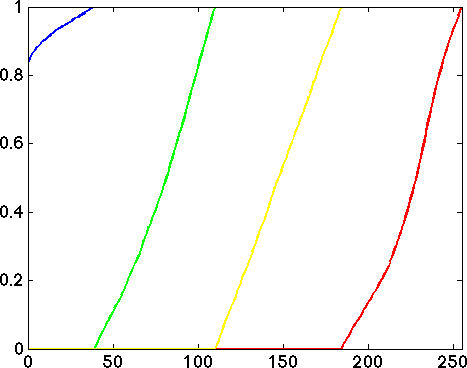}}
\subfigure[]{\includegraphics[width=4.5cm,height=3.1cm]{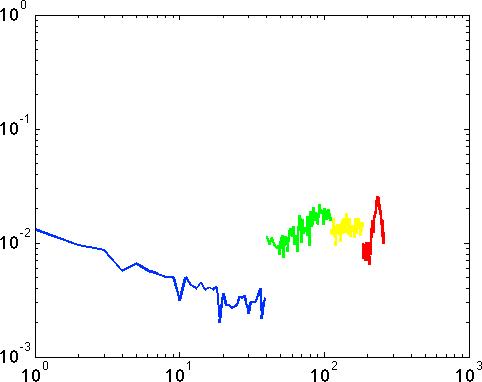}}\\
\caption{\footnotesize{Our fast four phase image segmentation model provides good segmentation results by distinguishes the gray matter from the surrounding white.
            First row:
            (a) Input image.
            (b) Segmentation result with $\lambda=1$.
            (c) Final binary segmentation $u_1$.
            (d) Final binary segmentation $u_2$.
            (e) Color visualization of the segmentation result.
            Second row:
            (g) Cumulative distribution function (CDF) of the four regions from (e).
            (g) Histogram of the four regions showing the separation clearly.}}\label{fig:our}
\end{figure}
\begin{figure}
\centering
\[\begin{array}{cccc}
\includegraphics[width=2.0cm,height=2.2cm]{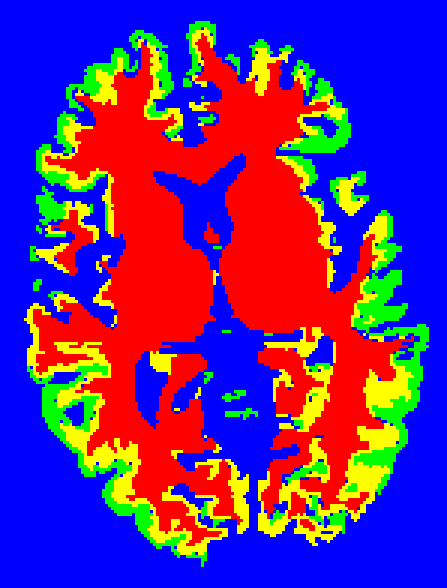}	&
\includegraphics[width=2.0cm,height=2.2cm]{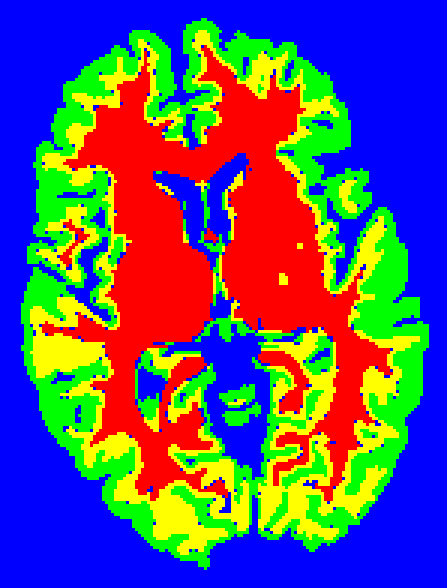}&
\includegraphics[width=2.0cm,height=2.2cm]{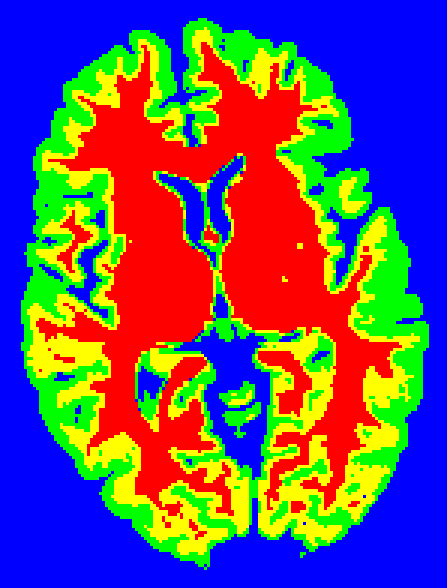}&
\includegraphics[width=2.0cm,height=2.2cm]{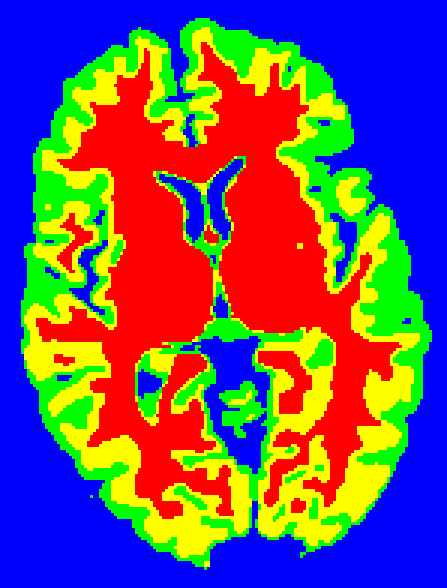}\\
\includegraphics[width=2.85cm,height=2.25cm]{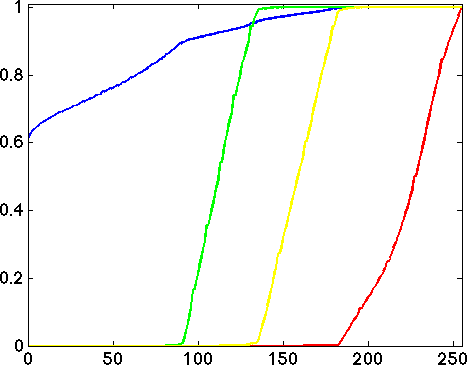}&
\includegraphics[width=2.85cm,height=2.25cm]{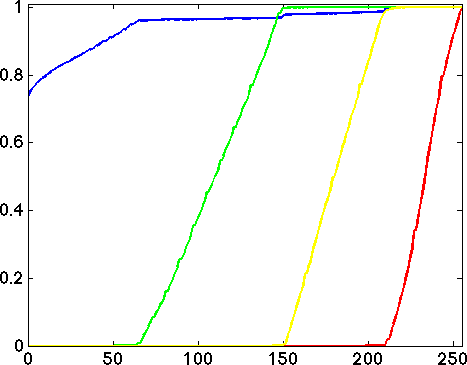}&
\includegraphics[width=2.85cm,height=2.25cm]{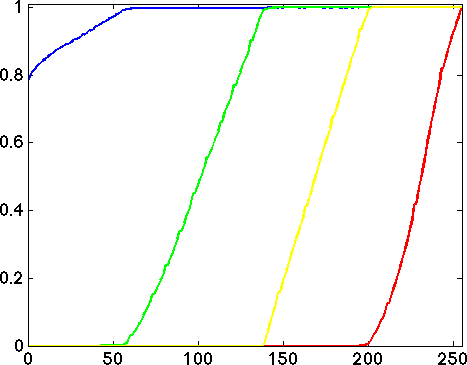}&
\includegraphics[width=2.85cm,height=2.25cm]{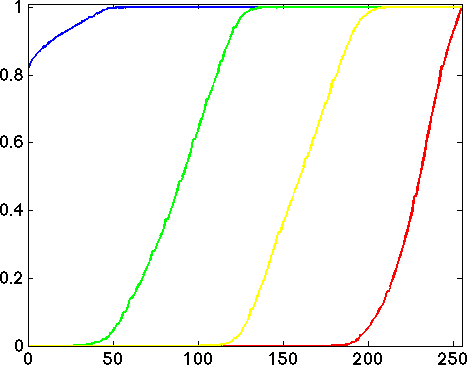}\\
\includegraphics[width=2.85cm,height=2.25cm]{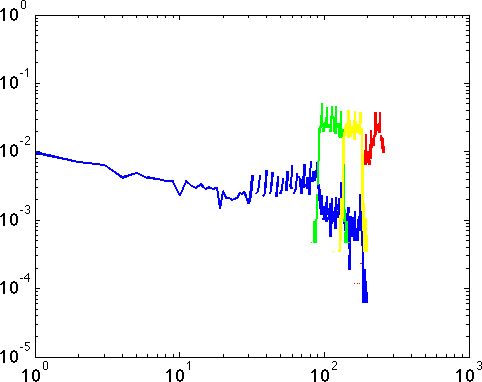}&
\includegraphics[width=2.85cm,height=2.25cm]{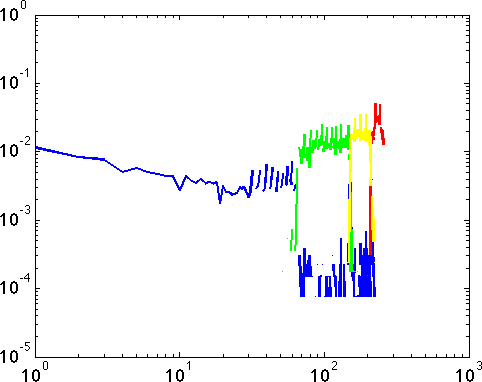}&
\includegraphics[width=2.85cm,height=2.25cm]{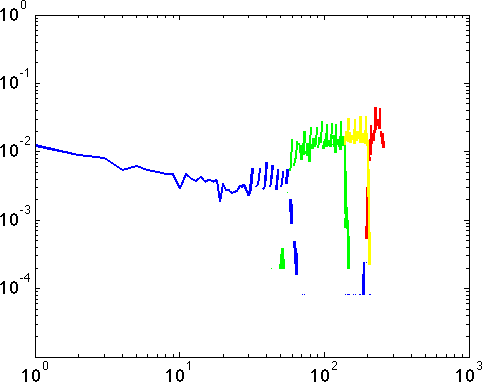}&
\includegraphics[width=2.85cm,height=2.25cm]{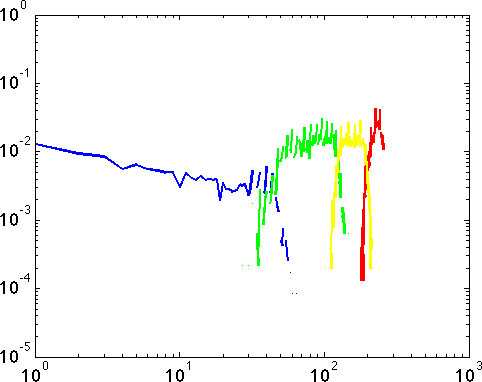}\\
\scriptsize{\textit{Ker}} & \scriptsize{\textit{Mean}} & \scriptsize{\textit{Clust}} & \scriptsize{\textit{Primal-Dual}}\\
\end{array}\]
\caption{\footnotesize{Comparison with \textit{Ker, Mean, Cluster, Mean, Primal-Dual} multiphase segmentation methods. 
First row: Color coded visualization of the obtained segmentation result.
Second row: Cumulative distribution function (CDF) of the four computed regions.
Third row: Histogram of the four regions showing the intersections.}}\label{fig:comp}
            \vspace{-0.4cm}
\end{figure}

We have used full brain MRI images available at the whole brain atlas~\footnote{\url{http://www.med.harvard.edu/aanlib/home.html}}. The parameters $\theta_1=\theta_2 = 0.001$ were fixed for the segmentation results reported here. In order to simplify notations we use $\lambda=\lambda_{11}=\lambda_{01}=\lambda_{10}=\lambda_{00}$ and we fix $\lambda=1$ in all our experiments as well. Equal weights are used for the four regions to be segmented as we do not want to introduce bias for certain phases. The images presented here are from $T1$ MRI imaging modality with slice thickness of $1$$mm$. Our scheme takes less than $0.2$ seconds (for $100$ iterations) on MATLAB2012a on a Mac laptop with Intel Core i7 CPU $2.3$GHz, $8$GB RAM CPU. Meanwhile, the average computation time for related models compared from the literature are in the region of $30$ seconds (for $100$ iterations) to converge to the final segmentation. 

Figure~\ref{fig:our} shows another example segmentation result of our globally convex four phase scheme. The noise (calculated relative to the brightest tissue, and denoted  by "$n$") is set to $3\%$ with intensity non-uniformity (denoted by "$RF$") is of strength $20\%$. The result of our four phase model is displayed in Figure~\ref{fig:our}(b) with two contours (Purple, Light-Blue) overlaid on top of the input image. Figure~\ref{fig:our}(c) and (d) show the two functions $u_1,u_2$ computed using our scheme and thresholded at $0.5$. The function $u_1$ captures the background shape~\ref{fig:our}(a) (corresponding to level set $\phi_1$) whereas function $u_2$ in Figure~\ref{fig:our}(b) (corresponding to level set $\phi_2$) contains the white matter. Figure~\ref{fig:our}(e) we use four different colors (Blue, Green, Yellow, and Maroon) to highlight different phases for better visualization of phase separation and boundary detection of regions. In Figure~\ref{fig:our}(f) and (g), we show cumulative distribution function (CDF) and the histogram of each of the four regions computed by the proposed method. The histograms highlight separation of different phases/regions indicating the superior performance of our splitting based numerical approach.  

Figure~\ref{fig:comp} shows a comparison result with other multiphase active contour methods from~\cite{AyedMitiche06,AyedMiticheBelhadj06,BenSalahMiticheTIP10,ChambollePock11} called in short, Ker, Mean, Clust and Primal-Dual respectively, for the same image in Figure~\ref{fig:our}(a). Note that to make a fair comparison with other models we used the same noise level and intensity non-uniformity for this example image.
In Figure~\ref{fig:comp} bottom two rows we show the cumulative distribution function (CDF) and histograms computed for each of the computed phases respectively. Compared with the histograms shown in Figure~\ref{fig:our}(f) and (g) for our scheme we see that proposed model provides better separation of regions. The histograms for the other schemes in Figure~\ref{fig:comp}(last row) show nontrivial intersections, highlighting the drawback in using level set based implementations. Moreover, the noise remains as speckles in the segmented regions whereas our model handles it efficiently. 

\subsection{Error metrics computation}

\begin{figure*}
\centering
\[\begin{array}{cc}
\includegraphics[width=.57in,height=.57in]{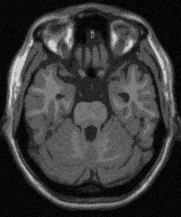}
\includegraphics[width=.57in,height=.57in]{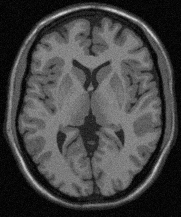}
\includegraphics[width=.57in,height=.57in]{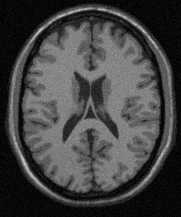}
\includegraphics[width=.57in,height=.57in]{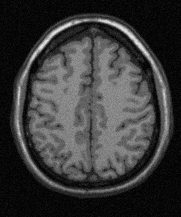}&
\includegraphics[width=.57in,height=.57in]{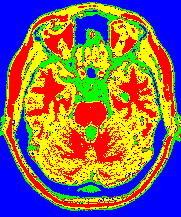}
\includegraphics[width=.57in,height=.57in]{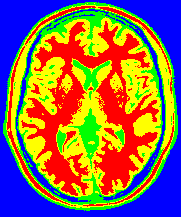}
\includegraphics[width=.57in,height=.57in]{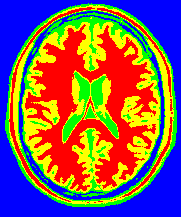}
\includegraphics[width=.57in,height=.57in]{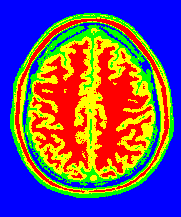}\\
\mbox{\scriptsize{Brain MRI images}}  & \mbox{\scriptsize{$n=3$, $RF=0$}} 
\end{array}\]
\[\begin{array}{cc}
\includegraphics[width=.57in,height=.57in]{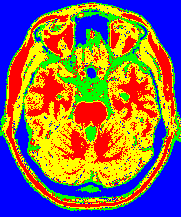}
\includegraphics[width=.57in,height=.57in]{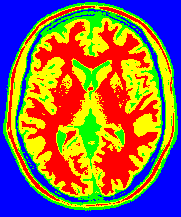}
\includegraphics[width=.57in,height=.57in]{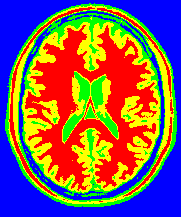}
\includegraphics[width=.57in,height=.57in]{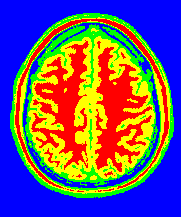}&
\includegraphics[width=.57in,height=.57in]{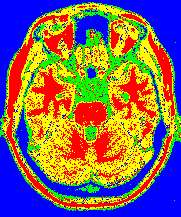}
\includegraphics[width=.57in,height=.57in]{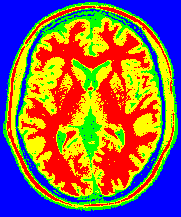}
\includegraphics[width=.57in,height=.57in]{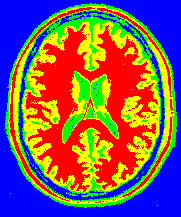}
\includegraphics[width=.57in,height=.57in]{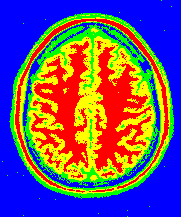}\\
\mbox{\scriptsize{$n=3$, $RF=20$}}  & \mbox{\scriptsize{$n=5$, $RF=0$}} 
\end{array}\]
\[\begin{array}{cc}
\includegraphics[width=.57in,height=.57in]{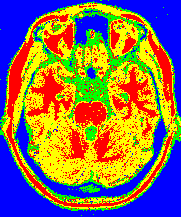}
\includegraphics[width=.57in,height=.57in]{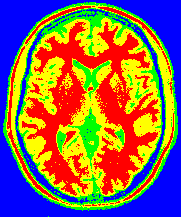}
\includegraphics[width=.57in,height=.57in]{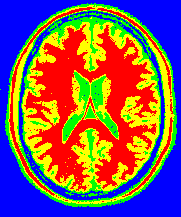}
\includegraphics[width=.57in,height=.57in]{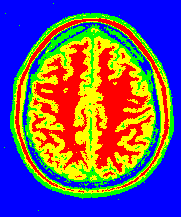}&
\includegraphics[width=.57in,height=.57in]{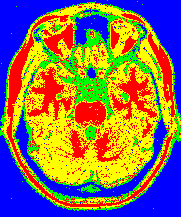}
\includegraphics[width=.57in,height=.57in]{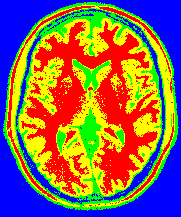}
\includegraphics[width=.57in,height=.57in]{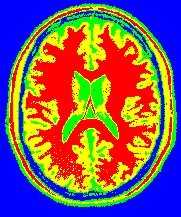}
\includegraphics[width=.57in,height=.57in]{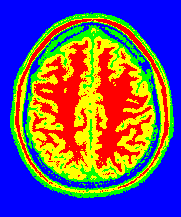}\\
 \mbox{\scriptsize{$n=5$, $RF=20$}}  & \mbox{\scriptsize{$n=5$, $RF=40$}} 
\end{array}\]
\caption{\footnotesize {Segmentation results for full Brain data-sets with representative axial slices. First (top-left) subfigure shows the noise-free brain MRI images. Next subfigures present segmentation results for different noise (``$n$'') and  non-uniformity (``$RF$") values for our scheme. Segmentation results are stable for increasing values of noise and intensity inhomegenities.}}\label{fig:slices}
\end{figure*}
We use the following quantitative error metrics to compare the schemes with gold standard ground truth segmentations. For more details about objective evaluation of image segmentation algorithms and for precise definitions of these metrics we refer to~\cite{UnniPAMI2007}.

\begin{itemize}
\item \textbf{DICE}:\\
The Dice coefficient~\cite{Dice45} is a popular error metric and is used to compare ground truth segmentation with those obtained with automatic multiphase segmentation schemes. By definition, for two binary segmentations $A$ and $B$, the Dice coefficient is computed as:
\begin{eqnarray}\label{E:dice}
D(A,B) = \frac{2\abs{A\cap B}}{\abs{A} + \abs{B}}.
\end{eqnarray}
Here the binary segmentation is computed automatically, using the segmentation curves and by thresholding regions obtained by all algorithms. The notation $\abs{A}$ denotes the number of pixels in the set $A$. Note that, a $D$ value of $1$ indicates perfect agreement. In particular, higher numbers indicate that the results of that particular scheme's result match the gold standard better than results that produce lower Dice coefficients. 

\item \textbf{RI}:\\
Rand Index: A metric based on a classical nonparametric test and is computed by counting pairs of pixels
that have compatible label relationships in the two segmentations to be compared.

\item \textbf{GCE}:\\
Global Consistency Error: A metric which computs the degree of overlap of the cluster associated with each pixel in one segmentation and its ÒclosestÓ approximation in the other segmentation. Values to closer to $0$ indicate better segmentation results.

\item \textbf{VI}:\\       
Variation of Information: A metric related to the conditional entropies between the class label distribution of the segmentations. This computes  a measure of information content in each of the segmentations and how much information one segmentation gives about the other. Values closer to $1$ indicate better segmentation results.

\end{itemize}
Note that all these metrics are for comparing two segmentations, one of which is assumed to be the available ground truth.
Table~\ref{tab:dice} shows the comparison of average Dice values (for $181$ images) of different models for different noise and intensity inhomogenieties taken from Brainweb database. As can be seen our scheme performs better in terms of the Dice coefficient compared with other related approaches. Similarly in Table~\ref{tab:measures} we see that the average RI, GCE and VI for different schemes against our model shows that the proposed globally convex multiphase scheme performs well overall.

Figure~\ref{fig:slices} shows representative segmentation results for full Brain data-sets (axial slices are shown) with different noise (``$n$'') and  non-uniformity (``$RF$") levels for our scheme. Different $n$ and $RF$ are specified in  Figure~\ref{fig:slices} for each row. This illustrates that our scheme preserves the topological changes as we move through the image stack. Moreover, our scheme can handle noise and intensity non-uniformity together effectively. Finally, in Figure~\ref{fig:slicecompar} we show different segmentation results for a particular image (slice number 79) taken across all noise and inhomogeniety levels for different schemes. The results indicate that Ker and Mean methods can lead to poor separation of different regions whereas noise can affect the result of Clust and Primal-Dual  schemes. Meanwhile, our  approach performs well and handles higher non-uniformity without degrading the final segmentation results. Further data-sets and extensive comparison results of all the schemes for full brain stacks are available online~\footnote{\url{http://dx.doi.org/10.6084/m9.figshare.781297}}. 

\begin{figure*}
\centering
\[\begin{array}{cc}
\includegraphics[width=1.21cm,height=1.46cm]{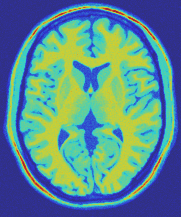}
\includegraphics[width=1.21cm,height=1.46cm]{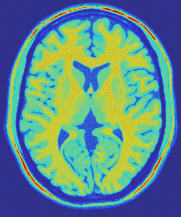}
\includegraphics[width=1.21cm,height=1.46cm]{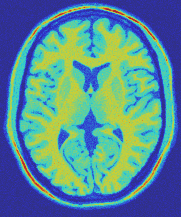}
\includegraphics[width=1.21cm,height=1.46cm]{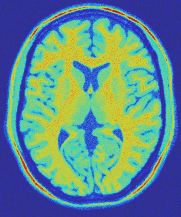}
\includegraphics[width=1.21cm,height=1.46cm]{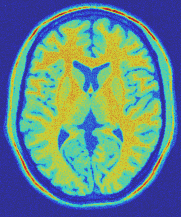}
&
\includegraphics[width=1.21cm,height=1.46cm]{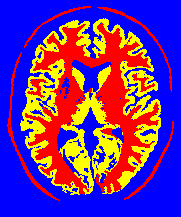}
\includegraphics[width=1.21cm,height=1.46cm]{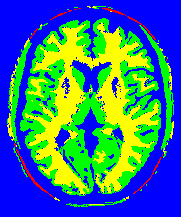}
\includegraphics[width=1.21cm,height=1.46cm]{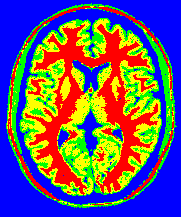}
\includegraphics[width=1.21cm,height=1.46cm]{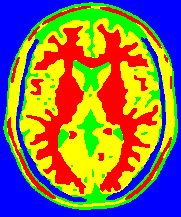}
\includegraphics[width=1.21cm,height=1.46cm]{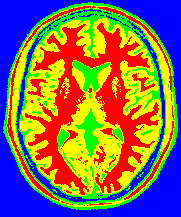}\\
\mbox{\scriptsize{Brain MRI image with different ($n$, $RF$) values}} & \mbox{\scriptsize{$n=3$, $RF=0$}}\\
\end{array}\]
\[\begin{array}{cc}
\includegraphics[width=1.21cm,height=1.46cm]{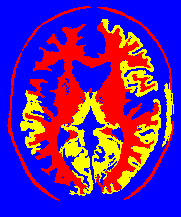}
\includegraphics[width=1.21cm,height=1.46cm]{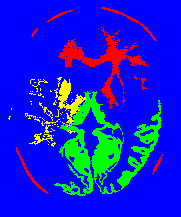}
\includegraphics[width=1.21cm,height=1.46cm]{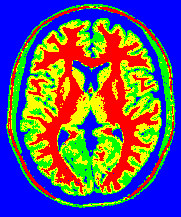}
\includegraphics[width=1.21cm,height=1.46cm]{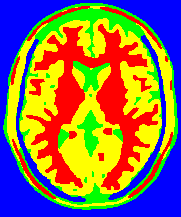}
\includegraphics[width=1.21cm,height=1.46cm]{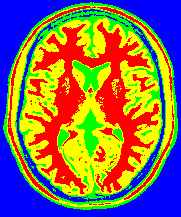}&
\includegraphics[width=1.21cm,height=1.46cm]{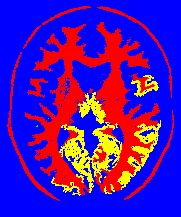}
\includegraphics[width=1.21cm,height=1.46cm]{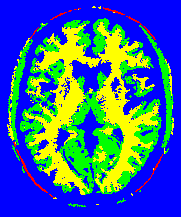}
\includegraphics[width=1.21cm,height=1.46cm]{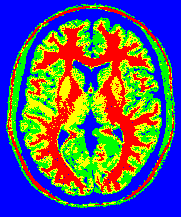}
\includegraphics[width=1.21cm,height=1.46cm]{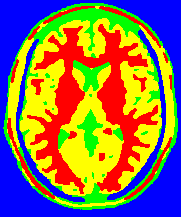}
\includegraphics[width=1.21cm,height=1.46cm]{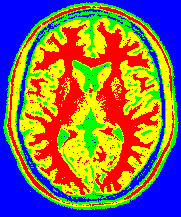}\\
\mbox{\scriptsize{$n=3$, $RF=20$}} & \mbox{\scriptsize{$n=5$, $RF=0$}}\\
\end{array}\]
\[\begin{array}{cc}
\includegraphics[width=1.21cm,height=1.46cm]{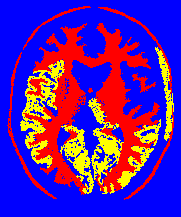}
\includegraphics[width=1.21cm,height=1.46cm]{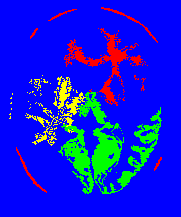}
\includegraphics[width=1.21cm,height=1.46cm]{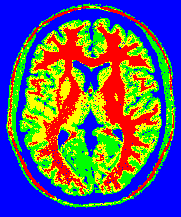}
\includegraphics[width=1.21cm,height=1.46cm]{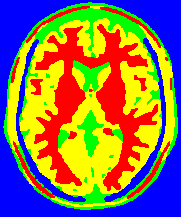}
\includegraphics[width=1.21cm,height=1.46cm]{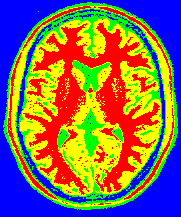}&
\includegraphics[width=1.21cm,height=1.46cm]{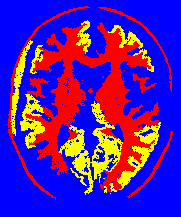}
\includegraphics[width=1.21cm,height=1.46cm]{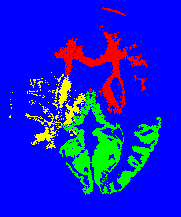}
\includegraphics[width=1.21cm,height=1.46cm]{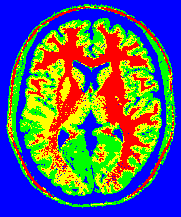}
\includegraphics[width=1.21cm,height=1.46cm]{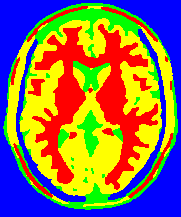}
\includegraphics[width=1.21cm,height=1.46cm]{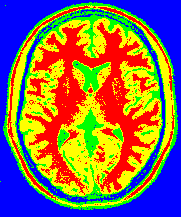}\\
\mbox{\scriptsize{$n=5$, $RF=20$}} & \mbox{\scriptsize{$n=5$, $RF=40$}}\\
\end{array}\]
\caption{\footnotesize{Comparison of color segmentation visualization for a single brain MRI image (Slice number $79$) with different ($n$, $RF$) levels. The top-left subfigure is different input images. Remaining subfigures contain different segmentation results. From left to right: Ker, Mean, Clust, Primal-dual, and Our approach, respectively.}}\label{fig:slicecompar}

\end{figure*}
\begin{table*}
\centering
\caption{\footnotesize{Average Dice coefficients values for different schemes in four different phases. Values near $1$ indicate the closeness of the segmentation to the ground truth segmentation. Best results are indicated by boldface.}}\label{tab:dice}
\footnotesize{
\begin{tabular}{ccccccccc}
    \hline
 $n$ & $RF$ & Regions & \textit{Ker} & \textit{Mean} & \textit{Clust} &  \textit{Primal-Dual}& \textit{Our}\\
\hline
  \multirow{4}{*} {$3$} &  \multirow{4}{*} {$0$} & 
        D1 & 0.305670 & 0.824283 & 0.886665 & 0.698128 &\textbf{0.944007}\\
 & & D2 & 0.224586 & 0.581326 & 0.419365 & 0.700223 &\textbf{0.915818}\\
 & & D3 & 0.131252 & 0.363182 & 0.110626 & 0.718244 & \textbf{0.870375}\\
 & & D4 & 0.565510 & 0.840712 & 0.693652 & 0.955584 & \textbf{0.965933}\\
\hline
 \multirow{4}{*} {$3$} &  \multirow{4}{*} {$20$} & 
           D1 & 0.306836 & 0.767452 & 0.873386 & 0.692621 &\textbf{0.931111}\\
 & &    D2 & 0.223917 & 0.534288 & 0.432176 & 0.695534 &\textbf{0.907063}\\
 & &    D3 & 0.130171 & 0.303436 & 0.110716 & 0.718225 & \textbf{0.873175}\\
 & &    D4 & 0.563268 & 0.823805 & 0.645526 & 0.953754 & \textbf{0.967594}\\
 \hline
 \multirow{4}{*} {$5$} &  \multirow{4}{*} {$0$} & 
           D1 &  0.305627 & 0.788663 & 0.877971 & 0.688774 &\textbf{0.912402}\\
 & &    D2 &  0.226171 & 0.539260 & 0.311053 & 0.680036 &\textbf{0.879323}\\
 & &    D3 &  0.126738 & 0.299720 & 0.100588 & 0.688052 & \textbf{0.829607}\\
 & &    D4 &  0.544178 & 0.807202 & 0.669333 & 0.948687 & \textbf{0.954868}\\

\hline
 \multirow{4}{*} {$5$} &  \multirow{4}{*} {$20$} & 
           D1 & 0.309830 & 0.746360 & 0.866684 & 0.708890 &\textbf{0.903028}\\
 & &    D2 & 0.225578 & 0.510120 & 0.318350 & 0.685553 &\textbf{0.870886}\\
 & &    D3 & 0.132034 & 0.253848 & 0.111433 & 0.683057 & \textbf{0.824065}\\
 & &    D4 & 0.539861 & 0.780266 & 0.613947 & 0.944511  & \textbf{0.953657}\\
\hline
 \multirow{4}{*} {$5$} &  \multirow{4}{*} {$40$} & 
           D1 & 0.310085 & 0.715360 & 0.825330 & 0.685356 &\textbf{0.872111}\\
 & &    D2 & 0.226430 & 0.478908 & 0.286386 & 0.678544 &\textbf{0.844355}\\
 & &    D3 & 0.130849 & 0.221841 & 0.127278 & 0.670710 &\textbf{0.806790}\\
 & &    D4 & 0.542354 & 0.766757 & 0.586870 & 0.543010 &\textbf{0.951143}\\
\hline
\end{tabular}
}
\end{table*}
\begin{table*}
\centering
\caption{\footnotesize{Average Rand Index (RI), Global Consistency Error (GCE) and Variation of Information (VI) for for different schemes. Best results are indicated by boldface.}}\label{tab:measures}
\footnotesize{
\begin{tabular}{ccccccccc}
    \hline
 $n$ & $RF$ & Error Metrics & \textit{Ker} & \textit{Mean} & \textit{Clust} & \textit{Primal-Dual} & \textit{Our}\\
\hline
  \multirow{3}{*} {$3$} &  \multirow{3}{*} {$0$} & 
              RI & 0.527025 & 0.849013 & 0.672026 & 0.895372 &\textbf{0.946341}\\
 & &  GCE & 0.332322 & 0.223942 & 0.158767 & 0.173467 &\textbf{0.085668}\\
 & &       VI & 2.483764 & 1.177137 & 1.490298 & 0.992642 & \textbf{0.569099}\\
 
\hline
  \multirow{3}{*} {$3$} &  \multirow{3}{*} {$20$} & 
             RI & 0.525251 & 0.833570 & 0.659371 & 0.887008 &\textbf{0.941475}\\
 & & GCE & 0.329012 & 0.236016 & 0.171285 & 0.189292 &\textbf{0.096797}\\
 & &      VI & 2.477927 & 1.255048 & 1.564992 & 1.065011 & \textbf{0.620833}\\

  \hline
 \multirow{3}{*} {$5$} &  \multirow{3}{*} {$0$} & 
             RI & 0.523815  & 0.825506 & 0.648725 & 0.884315 &\textbf{0.921506}\\
 & & GCE & 0.333999  & 0.250001& 0.154166  & 0.194308 &\textbf{0.135605}\\
 & &      VI & 2.507981& 1.326232 & 1.563033   & 1.087131 & \textbf{0.828043}\\

\hline
 \multirow{3}{*} {$5$} &  \multirow{3}{*} {$20$} & 
           RI  & 0.522536 & 0.813937 & 0.629603 & 0.877183 &\textbf{0.917244}\\
 & & GCE& 0.330010 & 0.252157 & 0.166120 & 0.209219 &\textbf{0.142290}\\
 & &     VI & 2.495524 & 1.375133 & 1.642941 & 1.148695 & \textbf{0.858660}\\

\hline
 \multirow{3}{*} {$5$} &  \multirow{3}{*} {$40$} & 
            RI & 0.520192 & 0.805794 & 0.604564 & 0.865636 &\textbf{0.905678}\\
 & &GCE & 0.326658 & 0.256492 & 0.176804 & 0.231084 &\textbf{0.163501}\\
 & &     VI & 2.489061 & 1.423080 & 1.737919 & 1.240130 &\textbf{0.950309}\\
 
\hline
\end{tabular}
}
\end{table*}
\section{Conclusion}\label{sec:concl}

We study a fast globally convex four phase active contour scheme for MRI image segmentation and provide a well posed convex energy minimization which can be used to determine piecewise constant segmentation without level sets. By using a dual minimization based implementation our approach provides better phase differentiation than other schemes. Experimental results on brain MRI images indicate the proposed approach provides better results compared with other active contour based multiphase segmentation schemes. 

Vector valued version similar to~\cite{CS00} is straightforward and our current implementation can handle RGB color images as well. Currently we are developing a three dimensional version for obtaining surface segmentations from MRI images similar to~\cite{DrapacaMultiphaseBrainMRI05} as well as a method to extract intensity non-uniformity patterns coupled with segmentations~\cite{Zhuge20091095}.
\bibliographystyle{plain}
\bibliography{endosrefs,mybrainrefs}
\end{document}